%% file: main.tex
\documentclass[sigconf]{acmart}
\AtBeginDocument{%
  \providecommand\BibTeX{{%
    \normalfont B\kern-0.5em{\scshape i\kern-0.25em b}\kern-0.8em\TeX}}}

\copyrightyear{2021}
\acmYear{2021} 
\setcopyright{iw3c2w3}
\acmConference[WWW '21]{Proceedings of the Web Conference 2021}{April 19--23, 2021}{Ljubljana, Slovenia} 
\acmBooktitle{Proceedings of the Web Conference 2021 (WWW '21), April 19--23, 2021, Ljubljana, Slovenia}
\acmPrice{}
\acmDOI{10.1145/3442381.3449927}
\acmISBN{978-1-4503-8312-7/21/04}

\usepackage{bm}
\usepackage{algorithm} 
\usepackage{algorithmic}

\newcommand{\ie}{\emph{i.e., }}
\newcommand{\eg}{\emph{e.g., }}
\newcommand{\etal}{\emph{et al.}}

\newcommand{\wrt}{\emph{w.r.t. }}

\newcommand{\Space}[1]{\mathbb{#1}}
\newcommand{\Set}[1]{\mathcal{#1}}

\begin{document}


\title{On the Equivalence of Decoupled Graph Convolution Network and Label Propagation}

\author{Hande Dong$^1$, Jiawei Chen$^{1*}$, Fuli Feng$^2$, Xiangnan He$^1$, Shuxian Bi$^1$, Zhaolin Ding$^3$, Peng Cui$^4$}

\affiliation{\institution{$^1$University of Science and Technology of China, $^2$National University of Singapore, \\
$^3$North Carolina State University, $^4$Tsinghua University.}} 

\email{donghd@mail.ustc.edu.cn, cjwustc@ustc.edu.cn, fulifeng93@gmail.com, xiangnanhe@gmail.com,} 
\email{stanbi@mail.ustc.edu.cn, zding8@ncsu.edu, cuip@tsinghua.edu.cn} 

\def\authors{Hande Dong, Jiawei Chen, Fuli Feng, Xiangnan He, Shuxian Bi, Zhaolin Ding, and Peng Cui}








\thanks{$*$ Jiawei Chen is the corresponding author.}

\renewcommand{\shortauthors}{\authors}


\begin{abstract}
The original design of Graph Convolution Network (GCN) couples \textit{feature transformation} and  \textit{neighborhood aggregation} for node representation learning. Recently, some work shows that coupling is inferior to decoupling, which supports deep graph propagation better and has become the latest paradigm of GCN (e.g., APPNP~\cite{klicpera_predict_2019} and SGCN~\cite{DBLP:conf/icml/WuSZFYW19}). Despite effectiveness, the working mechanisms of the decoupled GCN are not well understood. 

In this paper, we explore the decoupled GCN for semi-supervised node classification from a novel and fundamental perspective — label propagation. We conduct thorough theoretical analyses, proving that the decoupled GCN is essentially the same as the two-step label propagation: first, propagating the known labels along the graph to generate pseudo-labels for the unlabeled nodes, and second, training normal neural network classifiers on the augmented pseudo-labeled data. More interestingly, we reveal the effectiveness of decoupled GCN: going beyond the conventional label propagation, it could automatically assign structure- and model- aware weights to the pseudo-label data. This explains why the decoupled GCN is relatively robust to the structure noise and over-smoothing, but sensitive to the label noise and model initialization. Based on this insight, we propose a new label propagation method named Propagation then Training Adaptively (PTA), which overcomes the flaws of the decoupled GCN with a dynamic and adaptive weighting strategy. Our PTA is simple yet more effective and robust than decoupled GCN. 
We empirically validate our findings on four benchmark datasets, demonstrating the advantages of our method. The code is available at \url{https://github.com/DongHande/PT_propagation_then_training}. 

\end{abstract}




\maketitle

\input{1_introduction.tex}

\input{2_Preliminaries.tex}
\input{3_Understanding_Decoupled_GCN.tex}

\input{4_PTA.tex}

\input{5_experiments.tex}

\input{6_related_work.tex}

\input{7_conclusion.tex}
\begin{acks}
  This work is supported by the National Natural Science Foundation of China (U19A2079, U1936210) and the National Key Research and Development Program of China (2020AAA0106000).
\end{acks}
\bibliographystyle{ACM-Reference-Format}
\bibliography{main}

\appendix
\input{A_APPNP.tex}

\input{B_softmax.tex}

\input{C_concise_loss.tex}

\input{D_framework.tex}

\input{E_experimental_details.tex}
\input{F_struture_noise.tex}

\end{document}

%% file: 1_introduction.tex
\section{Introduction}

Graphs, which reflect relationships between entities, are ubiquitous in the real world, such as social, citation, molecules and biological networks. Recent years have witnessed the flourish of deep learning approaches on graph-based applications, such as node classification~\cite{DBLP:conf/aaai/LiHW18, DBLP:conf/ijcai/ChenGRHXGYZ19}, graph classification~\cite{yao2019graph, DBLP:conf/nips/KnyazevTA19}, link prediction~\cite{DBLP:conf/sigir/0001DWLZ020,DBLP:conf/nips/ZhangC18}, and community detection~\cite{jin2019graph}. Among various techniques, Graph Convolutional Network (GCN) has drawn recent attention due to its effectiveness and flexibility~\cite{kipf2017gcn, DBLP:conf/iclr/VelickovicCCRLB18,  DBLP:conf/iclr/XuHLJ19, DBLP:journals/Vikas/Generalization, DBLP:conf/ijcai/ZhuF0WLZZ20, sign_icml_grl2020, DBLP:conf/icml/YouYL19}. 

There are two important operations in a spatial GCN model: 1) \textit{feature transformation}, which is inherited from conventional neural networks to learn node representations from the past features, and 2) \textit{neighborhood aggregation} (also termed as \textit{propagation}), which updates the representation of a node by aggregating the representations of its neighbors.
In the original GCN~\cite{kipf2017gcn} and many follow-up models~\cite{DBLP:conf/iclr/VelickovicCCRLB18,DBLP:conf/ijcai/XieLYW020, DBLP:conf/www/Zhuang018}, the two operations are coupled, i.e., each graph convolution layer is consisted of both feature transformation and neighborhood aggregation. 
Nevertheless, some recent work find that such a coupling design is unnecessary and causes many issues such as training difficulty, hard to leverage graph structure deeply, and over-smoothing~\cite{klicpera_predict_2019, DBLP:conf/icml/WuSZFYW19, DBLP:conf/sigir/0001DWLZ020}. 
By separating the two operations, simpler yet more effective and interpretable models can be achieved, like the APPNP~\cite{klicpera_predict_2019}, SGCN~\cite{DBLP:conf/icml/WuSZFYW19}, DAGNN~\cite{DBLP:conf/kdd/LiuGJ20dagnn} for node classification, and LightGCN~\cite{DBLP:conf/sigir/0001DWLZ020} for link prediction. 
We term these models that separate the neural network from the propagation scheme as \textit{decoupled GCN}, which has become the latest paradigm of GCN. 
Despite effectiveness, the working mechanisms of the decoupled GCN are not well understood.


In this work, we strive to analyze the decoupled GCN deeply and provide insights into how it works for node classification. 
We prove in theory (by comparing the gradients) that the training stage of the decoupled GCN is essentially equivalent to performing a two-step label propagation. 
Specifically, the first step propagates the known labels along the graph to generate pseudo-labels for the unlabeled nodes, and then the second step trains (non-graph) neural network predictor on the augmented pseudo-labeled data. 
This novel view of label propagation reveals the reasons of the effectiveness of decoupled GCN: 

(1) The pseudo-labeled data serves as an augmentation to supplement the input labeled data. In semi-supervised graph learning settings, the labeled data is usually of small quantity, making neural network predictors have a large variance and easy to overfit. As such, the pseudo-labeled data  augmentation helps to reduce overfitting and improve the generalization ability. 

(2) Instead of assigning a uniform weight to the  pseudo-labeled data, decoupled GCN dynamically adjusts the weights based on the graph structure and model prediction during training. It follows two intuitions to adjust the weight of an unlabeled node: first, the node that is closer to the labeled source node is given a larger weight, and second, the node that has the pseudo-label different from the prediction of the neural network is given a smaller weight. The two intuitions are reasonable and explain why decoupled GCN is relatively robust to structure noise and over-smoothing.  

(3) To predict the label of a node in the inference stage, decoupled GCN combines the predictions of the node's $k$-hop neighbors rather than basing on the node's own prediction. Given the local homogeneity assumption of a graph, such an ensemble method effectively reduces the prediction variance, improving the accuracy and robustness of model prediction.

\begin{figure}[t!]
    \centering
    \includegraphics[width=0.45\textwidth]{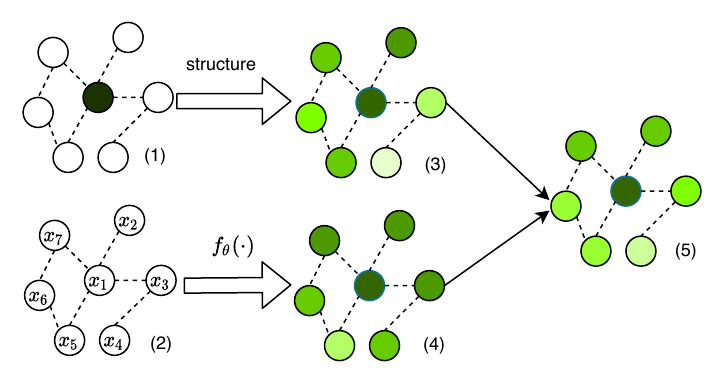}
    \caption{An illustration of the equivalent label propagation view of the decoupled GCN. 
    The dark node in (1) is a labeled source node in the training set. The first step propagates the label of this source node to its $K$-hop neighbors, 
    \ie pseudo-labeling neighbor nodes, 
    and the second step uses the augmented data to train the neural network predictor. 
    The color depth represents the weights of the pseudo-labeled data. 
    The final weights of pseudo-labels in (5) is determined by both graph structure (1)->(3) 
    and model prediction (2)->(4). 
    }
    \label{fig:train_decoupled_GCN}
\end{figure}

Figure \ref{fig:train_decoupled_GCN} illustrates the process.  
In addition to the advantages revealed by the equivalent label propagation view, we further identify two limitations of existing decoupled GCN:

(1) High sensitivity to model initialization. Since the weights of pseudo-labeled data are dynamically adjusted based on model prediction, the initialization of the model exerts a much larger impact on the model training. An ill initialization would generate incorrect predictions in the beginning, making the weights of pseudo-labeled data diverge from the proper values. As a result, the model may converge to an undesirable state. 

(2) Lacking robustness to label noise. The weights of the pseudo-labeled data generated from a labeled source node are normalized to be unity. It implies that different labeled nodes are assumed to have an equal contribution to weigh the pseudo-labeled data.
Such an assumption ignores the quality or importance of labeled data, which may not be ideally clean and have certain noises in practical applications. 


Our analyses provide deep insights into how the decoupled GCN works or fails, inspiring us to develop a better method by fostering strengths and circumventing weaknesses. Specifically, we propose a new method named \textit{Propagation then Training Adaptively} (PTA), which augments the classical label propagation with a carefully designed adaptive weighting strategy.  
After generating the pseudo-labeled data with label propagation, we dynamically control their weights during modeling training. Firstly, PTA abandons the equal-contribution assumption on labeled nodes, specifying the importance of a labeled node based on the consistence between its label and the predicted labels of its neighbors. Secondly, in the early stage of training, PTA reduces the impact of model prediction on the weighting of pseudo-labeled data, since an immature neural network usually yields unreliable prediction; as the training proceeds, PTA gradually enlarges the impact to increase the model's robustness to noise. Through the two designs, we eliminate the limitations of decoupled GCN and meanwhile make full use of its advantages. 


We summarize the contributions of this paper as below:
\begin{itemize}
    \item Conducting thorough theoretical analyses of decoupled GCN, proving that its training stage is equivalent to performing  a two-step label propagation.
    \item Analyzing the advantages and limitations of decoupled GCN from the label propagation perspective.
    \item Proposing a new label propagation method PTA that inherits the merits of decoupled GCN and overcomes its limitations with a carefully designed weighting strategy for pseudo-labeled data. 
    \item Validating our findings on four datasets of semi-supervised node classification and demonstrating the superiority of PTA in multiple aspects of accuracy, robustness, and stability.
\end{itemize}

%% file: 2_Preliminaries.tex
\section{Preliminaries}
In this section, we first formulate the node classification problem (Section 2.1). We then briefly introduce GCN (Section 2.2) and provide a summary of existing decoupled GCN (Section 2.3). Finally the classical label propagation is presented (Section 2.4).


\subsection{Problem Formulation}
Suppose we have a graph $G = (\Set V, \Set E)$, where $\Set V$ is the node set with $|\Set V|=n$ and $\Set E$ is the edge set. 
In this paper, we consider an undirected graph, whose adjacency matrix is represented with a matrix $\bm A \in \Space R^{n \times n} $. 
Let $\bm D=diag(d_1,d_2,...,d_n)$ denote the degree matrix of $\bm A$, where $d_i=\sum_{j\in \Set V} {a_{ij}}$ is the degree of the node $i$. 
The normalized adjacency matrix is represented as $\hat {\bm A}$. There are several normalization strategies, \eg $\hat {\bm A}=\bm {{
\bm D}}^{-1/2}\bm {{A}}\bm {{\bm D}}^{-1/2}$ \cite{DBLP:journals/corr/BrunaZSL13},  or $\hat {\bm A}=\bm D^{-1}\bm A$~\cite{DBLP:conf/aaai/ZhangCNC18}, 
or $\hat {\bm A}=\bm {\tilde{D}}^{-1/2} \bm {\tilde{A}}\bm {\tilde{D}}^{-1/2} $, $\bm {\tilde{A}} = \bm A + \bm I$~\cite{kipf2017gcn} where self-loop has been added. 
Also, we have features of the nodes, which are represented as a matrix $\bm X \in \mathbb{R} ^{n \times f}$. 
In this paper, we aim to solve graph-based semi-supervised node classification. 
Suppose that only a small set of nodes $i \in \Set V_l \subseteq \Set V$ have labels ${h(i)\in \Set C}$ from the label set $\Set C=\{1,2,...,C\}$. 
We also write the label $h(i)$ as a one-hot vector $\bm y_i\in \mathbb{R}^C$. Let $\bm Y$ denotes the observed label matrix, 
where $i$-th row is either set as $\bm y_i$ for $i\in \Set V_l$ or a zero vector otherwise. The goal is to predict the labels for the unlabeled nodes.

\newcommand{\tabincell}[2]{\begin{tabular}{@{}#1@{}}#2\end{tabular}}  
\begin{table}[t!]
    \centering
    \caption{Notation and Definitions.}
    \resizebox{.45\textwidth}{!}{%
    \begin{tabular}{ccccc}
    \hline
    Notation           &   Annotation     \\
    \hline
    $\hat {\bm A} $& the normalized adjacency matrix  \\
    $\bar{\bm A} $& $\sum \beta_k \hat {\bm A}^k$, the combination of $\hat {\bm A}^k,\, k =0, 1,2\cdots$ \\
    $\bm X \in \mathbb{R} ^{n \times f}$& the feature matrix of the nodes \\
    $\Set V_l, \Set V,\Set C$& the labeled node set, the universal node set, the label set \\
    ${h(i)\in \Set C}$& the label id of node $i$  \\
     $\bm y_i \in \mathbb{R}^C$ & the one-hot representation of the label of node $i$ \\
    $\bm Y$ & the observed label matrix \\
    $\bm Y_{soft}$& the soft label matrix using label propagation\\
    \hline
    \end{tabular}%
    }
    \label{notation}
\end{table}

\subsection{Graph convolution network (GCN)}
GCN performs layer-wise feature transformation and neighborhood aggregation. Each layer of GCN can be written as:
\begin{equation}
    \begin{aligned}
&{\bm H^{(k + 1)}} = \sigma \left( {\hat {\bm A}{\bm H^{(k)}}{\bm W^{(k)}}} \right),
 \end{aligned}
    \label{GCN}
\end{equation}
 where $\bm H^{(0)} = \bm X$, $\sigma(.)$ is an activation function such as ReLU; $\bm H^{(k)}$ and $\bm W^{(k)}$ denote the node representations and transformation parameters of the $k$-th layer. As we can see, each layer consists of two important operations: 1) neighborhood aggregation $\bm P^{(k)} = \hat {\bm A } \bm H^{(k)}$: which refines the representation of a node by aggregating the representations of its neighbors; 2) feature transformation $\bm H^{(k+1)}=\sigma (\bm P^{(k)} \bm W^{(k)})$, which is inherited from conventional neural networks to perform nonlinear feature transformation. By stacking multiple layers, GCN generates node representations by integrating both node features and the graph structure. 

\subsection{Decoupled GCN}
In this subsection, we first define a general architecture of decoupled GCN and show how it subsumes specific decoupled models~\cite{klicpera_predict_2019, DBLP:conf/kdd/LiuGJ20dagnn, DBLP:conf/icml/WuSZFYW19, DBLP:conf/sigir/0001DWLZ020}. We then transform it into a more concise formulation to facilitate theoretically analyzing it in the next section. 

In the original GCN \cite{kipf2017gcn} and many follow-up models ~\cite{DBLP:conf/iclr/VelickovicCCRLB18,DBLP:conf/ijcai/XieLYW020, DBLP:conf/www/Zhuang018}, neighborhood aggregation and feature transformation are coupled with each layer. Nevertheless, some recent works find that such a coupling design is unnecessary and propose to separate these two operations. 
These models can be summarized to a general architecture named \textit{decoupled GCN}:
\begin{equation}
    \begin{aligned}
\hat {\bm Y} = {\mathop{\rm softmax}\nolimits} \left( {\bar{\bm A}{\bm f_\theta }(\bm X)} \right),
 \end{aligned}
    \label{decouple}
\end{equation}
where $\bm f_\theta(\cdot)$ is a feature transformation function, which can be done by neural network. 
$\bar {\bm A} = \sum \beta_k \hat {\bm A}^k$ is determined by the graph structure and the propagation strategy, whose element reflects the proximity of two nodes in the graph. We will next show how this architecture subsumes existing decoupled methods.

\textit{APPNP and DAGNN.} APPNP~\cite{klicpera_predict_2019} claims several advantages of separating neural network modeling with the propagation scheme, formulating the model as: 
\begin{equation}
    \begin{aligned}
    &\bm H^{(0)} = \bm f_\theta (\bm X)\\
    &\bm H^{(k)} = (1-\alpha)\hat {\bm A} \bm H^{(k-1)} + \alpha \bm H^{(0)}, \quad k=1,2,...K-1\\
    &\bm {\hat Y}_{APPNP} = \operatorname{softmax}\left(\bm H^{(K)}\right).
    \end{aligned}
    \label{APPNP}
\end{equation}
APPNP uses Personalized PageRank~\cite{ilprints422} as the propagation strategy on graph. The model can be subsumed into Equation~\ref{decouple} by setting $\bar {\bm A}={(1 - \alpha )^K}{{\hat {\bm A}}^K} + \alpha \sum\nolimits_{k = 0}^{K - 1} {{{(1 - \alpha )}^k}{{\hat {\bm A}}^k}}$. The proof is presented in Appendix~\ref{APPNP_decoupled_GCN}. When $K \to \infty $, $\bar {\bm A}=\alpha(\bm I-(1-\alpha)\hat {\bm A})^{-1}$  is the diffusion kernel \cite{zhou2004learning}, which has been widely adopted to measure proximity in the graph. Another recent work DAGNN ~\cite{DBLP:conf/kdd/LiuGJ20dagnn} is similar to APPNP but uses a different propagation scheme: $\bar {\bm A} = \sum\nolimits_{k = 0}^K {{s_k}{{\hat {\bm A}}^k}}$, where $s_k$ controls the importance of different layers. 

\textit{SGCN and LightGCN.} The two models are designed by simplifying the original GCN to make it more concise and appropriate for downstream applications. 
SGCN and LightGCN are similar but for different tasks, so here we just present SGCN: 
\begin{equation}
    \begin{aligned}
\hat{\bm {Y}}_{\mathrm{SGCN}}=\operatorname{softmax}\left(\bm {S}^{K} \bm {X} \bm \Theta\right).
    \end{aligned}
    \label{SGCN}
\end{equation}
Naturally, it can be subsumed into Equation~\ref{decouple} by setting $\bm f_\theta(\bm X)=\bm X\bm \Theta$ and $\bar {\bm A}=\bm S^{K}$. 

In this paper, we would like to define a more concise formulation of decoupled GCN as follows:
\begin{equation}
    \begin{aligned}
\hat {\bm Y} =  {\bar {\bm A} {\bm f_\theta }(\bm X)},
  \end{aligned}
    \label{decou}
\end{equation}
where softmax function has been integrated into feature transformation function, \ie ${\bm f_\theta }(\bm X) \leftarrow \operatorname{softmax} ({\bm f_\theta }(\bm X))$. 
Although the predictions are not probabilities, integrating $\operatorname{softmax}(\cdot)$ into $\bm f_\theta$ does not affect model performance, since $\operatorname{softmax}(\cdot)$ is a monotonic function. Moreover, we analyze such an operation does not affect the model optimization in Appendix~\ref{softmax}. 
\subsection{Label Propagation}

Label propagation (LP)~\cite{zhu2002learning} is a classic semi-supervised learning algorithm that propagates the known labels along the graph to other unlabeled nodes. 
It can be formulated as follows:
\begin{equation}
    \begin{aligned}
    \bm Y^{(0)}&=\bm Y \\
    \bm Y^{(k)}&=\hat {\bm A} \bm Y^{(k-1)} \quad k=1,2,...,K-1\\
    \bm y_{i}^{(k)} &= \bm y_i,  \quad \forall i \in \Set V_l,
\end{aligned}
\label{LP1}
\end{equation}
where $\bm Y^{(k)}$ denotes the soft label matrix in iteration $k$, where each element $y^{(k)}_{ic}$ reflects how likely that the label of node $i$ is predicted as the label $c$. 
Similar to decoupled GCN, various propagation schemes can be adopted for LP, 
such as Personalized PageRank which has been adopted by APPNP where $\bm Y^{(k)}=(1-\alpha) \hat {\bm A} \bm Y^{(k-1)}+ \alpha \bm Y^{(0)}$. 
Similar to Equation~(\ref{decou}), the general formulation of LP can be summarized as follows: 
\begin{equation}
    \begin{aligned}
    {\bm Y}_{soft} = \hat {\bm Y} = \bar {\bm A}\bm Y.
    \end{aligned}
\label{LP}
\end{equation}

%% file: 3_Understanding_Decoupled_GCN.tex
\section{Understanding Decoupled GCN from Label prorogation}
In this section, we first define a simple training framework called \textit{Propagation then Training (PT)} (Section 3.1). We then conduct a rigorous mathematical analysis of  decoupled GCN to show that training a decoupled GCN is essentially equivalent to a case of PT (Section 3.2). 
Based on this insight, we further discuss the advantages and weaknesses of decoupled GCN (Section 3.3).

\subsection{Propagation then Training}
Before delving into decoupled GCN, we would like to design a simple model for node classification based on label propagation, which consists of two steps: 1) using the LP algorithm to propagate the known labels along the graph to generate pseudo-labels for the unlabeled nodes, and 2) training a neural network predictor on the augmented pseudo-labeled data. This training paradigm is simple and intuitive, using label propagation for data augmentation and benefiting model generalization. Formally, the model is optimized with the following objective function:
\begin{equation}
    \begin{aligned}
    L(\theta ) =\ell ({\bm f_\theta }(\bm X), \bar {\bm A}{\bm Y}),
\end{aligned}
\label{LPTT}
\end{equation}
where $\ell(\cdot)$ denotes the loss function between the predictions ${\bm f_\theta }(\bm X)$ and the soft labels $\bm Y_{soft}=\bar {\bm A}{\bm Y}$. For the node classification task, cross-entropy loss is most widely used~\cite{murphy2012machine}. Then the objective function for Equation~(\ref{LPTT}) can be re-written as follows: 
\begin{equation}
    \begin{aligned}
    L(\theta ) &= -\sum\limits_{i \in \Set V, k \in \Set C} \left(\sum\limits_{j\in \Set V_l} {{\bar a}_{ij}} {y_{jk}}\right)\log {f_{ik}} \\
    &=-\sum\limits_{i\in \Set V,j \in \Set V_l} {{\bar a}_{ij}}{\sum\limits_{k \in \Set C} {{y_{jk}}\log {f_{ik}}} } \\ 
    &= \sum\limits_{i\in \Set V,j \in \Set V_l} {{\bar a}_{ij}} \operatorname{CE} \left(\bm f_i, \bm y_j\right),
\end{aligned}
\label{LPTT1}
\end{equation}
where ${\bar a}_{ij}$ denotes the $(i,j)$-th element of the matrix $\bar {\bm A}$, $\operatorname{CE}(\cdot)$ denotes the cross entropy loss between the predictions $\bm f_{i}$ and the pseudo-labels $\bm y_j$. 
$\bm f_{i}$ represents the result from the function $\bm f_\theta(\bm x_i)$ with features $\bm x_i$ for node $i$ and $\bm f_{i,k}$ denotes its $k-$th element. 
The reformulated loss can be interpreted as using node $j$'s label to train the node $i$ 
and weighting this pseudo-labeled instance $(\bm x_i,\bm y_j)$ based on their graph proximity ${\bar a}_{ij}$. 
It is reasonable as the closer nodes in the graph usually exhibit more similarity. Larger ${\bar a}_{ij}$ suggests larger likelihood that the node $j$ is labeled as $\bm y_i$. 
These augmented data transfer the knowledge of the labeled nodes to their close neighbors. 
The weight of instance $(\bm x_i,\bm y_j)$ is static ${\bar a}_{ij}$, \ie staying the same value during training, 
so we name this specific model as \textit{Propagation then Training Statically (PTS)}. 

In most cases, static weighting is unsatisfied (which will be discussed in next subsection). Thus, we would like to extend the PTS to a more general framework with flexible weighting strategies. We name the framework as \textit{Propagation then Training (PT)}, which optimizes the following objective function:
\begin{equation}
    L_{PT}=L(\theta ) = \sum\limits_{i\in \Set V,j \in \Set V_l} w_{ij} \operatorname{CE} \left(\bm f_i, \bm y_j\right),
\end{equation}
where $w \mathop  = \limits^{\operatorname{define}}  g(\bm f(X), \bm A)$ means a general weighting strategy, which is controlled by the model prediction and propagation scheme with a specific function $g$.

\subsection{Connection between Decoupled GCN and PT}
In this subsection, we conduct a mathematical analysis of the gradients of decoupled GCN, proving that the training stage of decoupled GCN is essentially equivalent to a special case of PT. 
In fact, we have the following lemma:
\newtheorem{lem}{Lemma}
\begin{lem}
\label{la1}
Training a decoupled GCN is equivalent to performing \textit{Propagation then Training} with 
dynamic weight $w_{ij} = \frac{\bar a_{ji} f_{i,h(j)}}{{\sum\limits_{q \in \Set V} {{{\bar a}_{jq}}{f_{q,h(j)}}}}}$  for each pseudo-labeled data $(\bm x_i,\bm y_j)$, 
where $h(j)$ means the label ID of node j, \ie $y_{j, h(j)} = 1$.
\end{lem}
\begin{proof}
The lemma can be proved by comparing the gradients of decoupled GCN and PT. Adopting cross entropy loss, the loss function of decoupled GCN is: 
\begin{equation}
    \begin{aligned}
{L_{DGCN}}& = \ell(\bar {\bm A}{\bm f_\theta }(\bm X),\bm Y)\\
 &= -\sum\limits_{j \in {\Set V_l}, k \in \Set C} {{{y_{jk}}(\log \sum\limits_{i \in \Set V} {{{\bar a}_{ji}}{f_{ik}}} )} }. \\
\end{aligned}
\label{lem0}
\end{equation}
The gradients of the objective function \wrt $\theta$ can be written as:
\begin{equation}
    \begin{aligned}
{\nabla _\theta }{L_{DGCN}}& = -\sum\limits_{j \in {\Set V_l},k \in \Set C} {{{y_{jk}}{\nabla _\theta }(\log \sum\limits_{i \in \Set V} {{{\bar a}_{ji}}{f_{ik}}} )} } \\
 &=-\sum\limits_{j \in {\Set V_l},k \in \Set C} {{y_{jk}} { {\frac{\sum\limits_{i \in \Set V}{{{\bar a}_{ji}}{\nabla _\theta }{f_{ik}}}}{{\sum\limits_{q \in \Set V} {{{\bar a}_{jq}}{f_{qk}}} }}} } }.\\
\end{aligned}
\label{lem1}
\end{equation}
As $\bm y_j$ is an one-hot vector, only the $h(j)$-th element of $\bm y_j$ equal to one. The gradients can be rewritten as follows: 
\begin{equation}
    \begin{aligned}
{\nabla _\theta }{L_{DGCN}}& = -\sum\limits_{j \in {\Set V_l}} {{y_{j,h(j)}} { {\frac{\sum\limits_{i \in \Set V}{{{\bar a}_{ji}}{\nabla _\theta }{f_{i,h(j)}}}}{{\sum\limits_{q \in \Set V} {{{\bar a}_{jq}}{f_{q,h(j)}}} }}} } }\\
 &=-\sum\limits_{i \in \Set V,j \in {\Set V_l}} \frac{\bar a_{ji} f_{i,h(j)}}{{\sum\limits_{q \in \Set V} {{{\bar a}_{jq}}{f_{q,h(j)}}}}} {y_{j,h(j)}} \frac{{\nabla _\theta }{f_{i,h(j)}}}{{f_{i,h(j)}}} \\ 
 &=\sum\limits_{i \in \Set V,j \in {\Set V_l}} \frac{\bar a_{ji} f_{i,h(j)}}{{\sum\limits_{q \in \Set V} {{{\bar a}_{jq}}{f_{q,h(j)}}}}} {\nabla _\theta } \operatorname{CE} \left( \bm f_i,\bm y_j \right).
\end{aligned}
\label{lem2}
\end{equation}
Note that the gradients of the PT \wrt $\theta$ is:
\begin{equation}
    \begin{aligned}
{\nabla _\theta}{L_{PT}}&=\sum\limits_{j\in {\Set V_l},i \in \Set V}  w_{ij} {\nabla _\theta } \operatorname{CE} \left( \bm f_i,\bm y_j \right).
\end{aligned}
\label{lem3}
\end{equation}
Comparing Equation~(\ref{lem1}) with Equation~(\ref{lem3}), the decoupled GCN is a special case of LP by setting $w_{ij} $ as $ \frac{\bar a_{ji} f_{i,h(j)}}{{\sum\limits_{q \in \Set V} {{{\bar a}_{jq}}{f_{q,h(j)}}}}}$. 
\end{proof}

\subsection{Analyzing Decoupled GCN from PT}
Based on above proof, the working mechanism of decoupled GCN can be well understood. 
It is equivalent to performing a label propagation to generate pseudo-labels and then optimizing the neural predictor on the pseudo-labeled data with a weighted loss function. 
As we can see from Equation (\ref{lem2}), the essence of decoupled GCN is to construct more training instances with label propagation. 
Specifically, it propagates the known label of node $j$'s to other nodes (such as node $i$) and uses the augmented pseudo-label data ($\bm x_i,\bm y_j$) to learn a better classifier. 
Moreover, these training instances are weighted by the product of the graph proximity $\bar a_{ji}$ and the model prediction $f_{i,h(j)}$. This finding is highly interesting and helps us to understand the reasons of the effectiveness of decoupled GCN:

(S1) Label propagation serves as a data augmentation strategy to supplement the input labeled data. In semi-supervised node classification task, the labeled data is usually of small quantity, making it insufficient to train a good neural predictor. The model would have a large variance and easily sink into over-fitting. As such, the pseudo-labeled data augmentation helps to reduce overfitting and improves model performance.

(S2) Instead of assigning a static weight to the pseudo-labeled data, decoupled GCN dynamically adjusts the weight. On the one hand, the node that is closer to the labeled source node is given a larger weight $\bar a_{ji}$. As analyzed in Section 4.1, this setting matches our intuition, as closer nodes usually exhibit more similarity in their properties or labels. On the other hand, if a node has a pseudo-label highly different from the prediction, the pseudo-labeled data would obtain smaller weight. This setting makes the model more robust to the structure noise, which is common in real-world graph topology. Real-world graphs may not be ideally clean and have certain noises in edges. Two connected nodes sometimes exhibit different properties and belong to different classes. For example, in a social network, connected friends may belong to different schools; in a citation network, a paper may cite the work from other fields. The pseudo-labels propagated along such noisy inter-class edges are unreasonable and may hurt the model performance. In decoupled GCN, this bad effect could be mitigated, as the contribution of these unreliable pseudo-labeled data will be reduced. Moreover, this setting endows the model the potential to mitigate over-smoothing. The noisy signal from distant labeled nodes would be attenuated heavily by the weights. 

(S3) To predict the label of a node in the inference stage, as shown in Figure \ref{ensemble}, decoupled GCN combines the predictions of the node's $K$-hop neighbors rather than basing only on the node's own prediction. 
For a target node $i$, decoupled GCN finds its $K$-order neighbors and combines their predictions with the trained neural network. Such an ensemble mechanism further boosts model's performance by reducing the prediction variance. 

\begin{figure}[t!]
    \centering
    \includegraphics[width=0.4\textwidth]{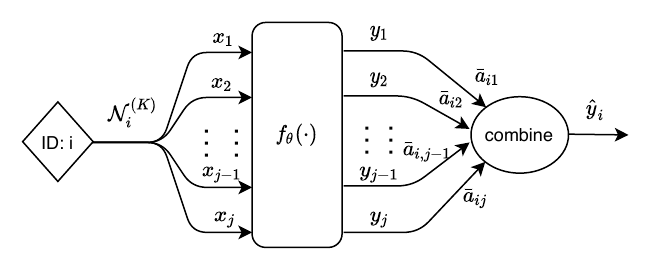}
    \caption{APPNP ensembles the predictions of neighbors to generate final prediction in the inference phase. }
    \label{ensemble}
\end{figure}

However, some weaknesses of the decoupled GCN are revealed.

(W1) Since the weights of pseudo-labeled data are dynamically adjusted based on model prediction, the initialization of the model exerts a much larger impact on the model training. An ill initialization
would generate incorrect predictions in the beginning, making the weights of pseudo-labeled data diverge from the proper values. It further skews the contribution of pseudo-labeled data and the model may converge to undesirable states. 

(W2) The weights of the pseudo-labeled data generated from a labeled source node are normalized to be unity, \ie $\sum_{i \in \Set V}w_{ij} = 1$. It implies that different labeled nodes are assumed to have an equal contribution
to weigh the pseudo-labeled data. Such an assumption ignores the quality or importance of labeled data, which may not be ideally clean and have certain noises in practical applications. Some labels are carefully labeled by expert, whereas some labels may be contaminated by accidents or envirnmental noises. Treating them equally is not reasonble. As such, the model is vulnerable to label noises which will deteriorate the model performance.

%% file: 4_PTA.tex
\section{Proposed method: Propagation Then Training Adapatively}
Given the analysis of decoupled GCN, in this section, we aim to develop a better method that can foster its merits and overcome its weaknesses. We name the proposed method as \textit{Propagation then Training Adaptively (PTA)}, which improves the LP view of decoupled GCN with an adaptive weighting strategy. Focusing on the two weaknesses of decoupled GCN, PTA makes two revisions:

(1) To make the model more robust to label noise, we remove the normalization of weights of decoupled GCN to let different labeled data exert varying impact on model training. The weight without normalization can be written as: 
\begin{equation}
    \begin{aligned}
w_{ij} = {\bar a_{ji} f_{i,h(j)}}.
    \end{aligned}
\end{equation}
The accumulated weight of the pseudo-labeled data generated from a specific labeled source node can be written as: 
\begin{equation}
    \begin{aligned}
S_{j}={\sum\limits_{i \in V} {{{\bar a}_{ji}}{f_{i,h(j)}}} }.
    \end{aligned}
\end{equation}
which is dynamically adjusted according to ${f_{i,h(j)}}$. 
Specifically, the importance of each labeled data $S_{j}$ is determined by the consistence between its label and the predicted labels of its neighbors (or multi-hop neighbors). When the labels of its neighbors are predicted to be different from $h(j)$ (\ie $f_{i,h(j)}$ is small), it implies the labeled data may be contaminated by noises. Naturally, this design reduces the contribution of this unreliable labeled data and makes the model more robust to label noise. 
Another advantage of removing normalization is making the model more concise and easier to implement,  as the computationally expensive summation over the neighbors (or multi-hop neighbors) is avoided. Nevertheless, this design may increase model's sensitivity to the model initialization which also determines the impact of the labeled data. We address this issue in the next design. 

(2) The sensitivity to model initialization is caused by the weights $f_{i,k(j)}$.
However, blindly removing $f_{i,k(j)}$ is problematic as it would hurt model robustness to label noise and structure noise.
To deal with this problem, we develop an adaptive weighting strategy as follows:
\begin{equation}
    \begin{aligned}
    w_{ij} = \bar a_{ji} f_{i,h(j)}^\gamma ,\quad \quad \gamma  = \log (1 + e/\epsilon), \\
    \end{aligned}
\label{pta}
\end{equation}
 where $\gamma$ is define to control the impact of $f_{i,h(j)}$ on the weighting of the pseudo-labeled data, which will evolve with the training process. $e$ denotes the current training epoch and $\epsilon$ is a temperature hyper-parameter controlling the sensitivity of $\gamma$ to $e$. Our design is simple but rather effective. In the early stage of training, when the immature neural network generates relatively unreliable prediction, PTA reduces the impact of model prediction on weighting pseudo-labeled data. This setting makes the model yield stable results. With the training proceeding, as the neural predictor gradually gives more accurate results, PTA enlarges its impact to make the model robust to both label noise and structure noise.
 
 
Through the two designs, we eliminate the limitations of decoupled GCN and meanwhile make full use of its advatanges. Overall, PTA optimizes the following objective function:
\begin{equation}
    \begin{aligned}
L_{PTA}(\theta ) = \sum\limits_{i\in V,j \in V_l} w_{ij} \operatorname{CE} \left(\bm f_i, \bm y_j\right) ,\quad \quad {w_{ij}} =\bar a_{ji}{f^\gamma_{i,h(j)}}.
   \end{aligned}
\end{equation}
We also give a concise matrix-wise formula as follows: 
\begin{equation}
    \begin{aligned}
  L_{PTA}(\theta ) &   = - \mathop{SUM} \left( \bm Y_{soft} \otimes \bm f(\bm X)^\gamma \otimes log\left(\bm f_\theta(\bm X)\right)\right),\\
   \end{aligned}
   \label{concise_form}
\end{equation}
where $\otimes$ represents the element-wise product, $\mathop{SUM}(\cdot)$ represents the sum of all elements in matrix,  $\bm Y_{soft}$ represents the soft label generated using label propagation, $\bm f(\bm X)$ does not propagate gradients backward, and gradients propagate only from $log\left(\bm f_\theta(\bm X)\right)$. The proof for Equation~\ref{concise_form} can be seen in Appendix~\ref{concise_loss}. 
The complete framework of PTA is in Appendix~\ref{framework}. 

Compared the concise form of PTA with vanilla decoupled GCN (\eg APPNP), PTA also has an advantage of efficient computation. 
In decoupled GCN, both feature transformation and neighbor aggregation need to be conducted in each epoch. 
Moreover, the updating of the transformation function $\bm f_\theta(\bm X)$ requires back propagation along both operations. 
The complexity of decoupled GCN is $\mathcal{T}(\bm f(\bm X))+\mathcal{T}(\bar {\bm A}\bm H)$, 
where $\mathcal{T}(\bm f(\bm X))$ and $\mathcal{T}(\bar {\bm A}\bm H)$ denote the complexity of the two operations, respectively. 
In our PTA, the two operations have been thoroughly separated, \ie graph propagation can be pre-processed and used for all epochs while each training epoch only need to consider feature transformation. 
The overall training algorithm of our PTA is presented in Algorithm \ref{PTA}. As we can see, we just need to run the label propagation to generate soft-label matrix $\bm Y_{soft}$, 
which can be considered as data pre-processing. We then train the neural predictor $\bm f_\theta(\bm X)$ with soft label $\bm Y_{soft}$ and additional weighting ${\bm f(\bm X)}^\gamma$, 
avoiding the expensive forward and back propagation of the neighbor aggregation in each epoch. 
The complexity of PTA is reduced from $\mathcal{T}(\bm f(\bm X))+\mathcal{T}(\bar {\bm A}\bm H)$ to $\mathcal{T}(\bm f(\bm X))$. 

It is worth mentioning that the neighborhood aggregation mechanism of GCN is non-trivial to implement in parallel~\cite{DBLP:conf/ipps/TianMYD20, DBLP:journals/tsipn/ScardapaneSL21}, making it the efficiency bottleneck on large graphs. In contrast, the label propagation used in PTA is well studied and much easier to implement on large graphs. As such, our PTA eases the implementation of graph learning and has great learning potentials to be deployed in large-scale industrial applications. 



\begin{algorithm}[t]
\caption{Propagation then Training Adaptively. }
\label{PTA}
\begin{algorithmic}[0] 
\REQUIRE ~~
Graph $G = (\Set V, \Set E)$; Features $\bm X$; Observed labels $\bm Y$;
\ENSURE ~~
Neural network predictor $\bm y=\bm f_\theta(\bm x)$ . \\
\STATE Generate soft label matrix $\bm {\hat Y}_{soft}$ with label propagation: 
\FOR{$e$ = $1$ to $Epoch_{max}$} 
\STATE Calculate the adaptive factor: $\gamma = log(1+e/\varepsilon)$
\STATE Calculate loss: $L=-\mathop{SUM} (\bm Y_{soft}\otimes \bm f(\bm X)^\gamma \otimes log(\bm f_\theta(\bm X)))$
\STATE Optimize $\theta$ by minimizing $L + \lambda L_{reg}$ with gradient descent
\ENDFOR
\STATE \textbf{return} Neural network predictor $\bm y=\bm f_\theta(\bm x)$
\end{algorithmic}
\end{algorithm}

%% file: 5_experiments.tex
\section{experiments}
We conduct experiments on four real-world benchmark datasets to evaluate the performance of existing decoupled methods and our proposed PTA. We aim to answer the following research questions: 
\begin{itemize}
    \item [\textbf{RQ1:}] Do three advantages of decoupled GCN that we analyzed from LP perspective indeed boost its performance? 
    \item [\textbf{RQ2:}] Is decoupled GCN sensitive to the initialization and the label noise? How does PTA overcome these problems?
    \item [\textbf{RQ3:}] How does the proposed PTA perform as compared with state-of-the-art GCN methods? 
    \item [\textbf{RQ4:}] Is PTA more efficient than decoupled GCN? 
\end{itemize}

\subsection{Experimental Setup}
We take APPNP as the representative model of decoupled GCN for experiments, since APPNP performs similarly with DAGNN while they are both superior over SGCN. To reduce the experiment workload and keep the comparison fair, we closely follow the settings of the APPNP work~\cite{klicpera_predict_2019}, including the experimental datasets and various implementation details. 
\paragraph{Datasets} Following the APPNP work~\cite{klicpera_predict_2019}, we also use four node-classification benchmark datasets for evaluation, including CITESEER~\cite{sen2008collective}, 
CORA\_ML~\cite{mccallum2000automating}, PUBMED~\cite{namata2012query} and MICROSOFT ACADEMIC~\cite{shchur2018pitfalls}. 
CITESEER, CORA\_ML, and PUBMED are citation networks, where each node represents a paper and an edge indicates a citation relationship. 
MICROSOFT ACADEMIC is a co-authorship network, where nodes and edges represent authors and co-author relationship, respectively. 
The dataset statistics is summarized in Table~\ref{Dataset}. Also, we follow \cite{klicpera_predict_2019} and split each dataset into training set, early-stopping set, and test set, where each class has 20 labeled nodes in the training set. 
\begin{table}[t]
    \centering
    \caption{Statistics of the datasets.}
        \begin{tabular}{ccccc}
        \hline
        Dataset     & Nodes  & Edges   & Features & Classes  \\
        \hline
        CITESEER    & 2,110  & 3,668   & 3,703    & 6        \\
        CORA\_ML        & 2,810  & 7,981   & 2,879   & 7      \\
        PUBMED      & 19,717 & 44,324  & 500      & 3          \\
        MS\_ACADEMIC      & 18,333 & 81,894  & 6,805      & 15         \\
        \hline
        \end{tabular}
    \label{Dataset}
\end{table}


\paragraph{Compared methods}
The main competing method is APPNP, which has shown to outperform several graph-based models, including Network of GCNs (N-GCN)~\cite{DBLP:conf/uai/Abu-El-HaijaKPL19}, graph attention network (GAT) \cite{DBLP:conf/iclr/VelickovicCCRLB18}, jumping knowledge networks with concatenation (JK) ~\cite{xu2018representation}, etc. As the comparison is done on the same datasets under the same evaluation protocol, we do not further compare with these methods. In addition to APPNP, we further compare with two relevant and competitive decoupled methods: DAGNN~\cite{DBLP:conf/kdd/LiuGJ20dagnn} and SGCN~\cite{DBLP:conf/icml/WuSZFYW19}, and two baseline models MLP, GCN~\cite{kipf2017gcn}. Further two variants of PTA are tested: 
\begin{itemize}
\item \textbf{PTS}: propagation then training statically, 
where we give a graph-based static weighting $a_{ji}$ to pseudo-labeled data. PTS can be considered as simple version of APPNP and PTA, where model-based weighting is removed. More details about PTS can refer to section 3.1.
\item \textbf{PTD}: propagation then training dynamically. 
PTD can be considered as an intermediate model between APPNP and PTA, where each pseudo-labeled data is weighed by $a_{ji}f_{i,h(j)}$. Comparing with APPNP, weighting normalization has been removed; Comparing with PTA, PTD does not use the adaptive factor. 
\end{itemize}

\paragraph{Implementation details}
In our experiments, we take APPNP as the representative model of decoupled GCN for experiments. For fair comparison, all the configures of PTD, PTS, PTA are the same as APPNP including layers, hidden units, regularization, early stopping, initialization and optimizer. Also, we use the same propagation scheme as APPNP. That is, we use PageRank to propagate labels with the same hyper-parameters. The additional parameter $\varepsilon$ in Equation~(\ref{pta}) is set as 100 across all datasets. We run each experiment 100 times on multiple random splits and initialization. More details of the setting of our PTA are presented in Appendix~\ref{detail}. 
For the compared methods, we refer to their results or settings that are reported in their papers. 
We will share our source code when the paper gets published.


\subsection{Empirical Analyses of APPNP (RQ1)}
In this subsection, we take APPNP as a representative method to answer the RQ1. As discussed in section 3.3, we attribute the effectiveness of decoupled GCN to three different aspects: 1) data augmentation through label propagation, 2) dynamic weighting, and 3) Ensembles multiply predictors. To show the impact of these aspects, we conduct ablation studies and compare APPNP with its four variants: 1) MLP, where both pseudo-labels and ensemble are removed; 2) APPNP-noa-nof, where weighting of pseudo-labels ($w_{ij}$) is removed; 3) PTS, where dynamic weighting $f_{i,h(j)}$ is removed. 4) APPNP-noe, where the ensemble is removed; The characteristics of these compared methods are presented in Table~\ref{ab_appnp}. 

\begin{table}[t!]
    \centering
    \caption{Characteristics of APPNP and its variants.}
    \resizebox{.45\textwidth}{!}{%
    \begin{tabular}{ccccc}
    \hline
    Method  & \begin{tabular}[c]{@{}c@{}}Pseudo-\\ labels?\end{tabular}  &  \begin{tabular}[c]{@{}c@{}}Graph-based\\ weighting?\end{tabular}   & \begin{tabular}[c]{@{}c@{}}Model-based\\ weighting?\end{tabular}   & Ensemble?    \\
    \hline
    MLP& $\times$ &$\times$&$\times$ & $\times$\\
    APPNP-noa-nof&\checkmark&$\times$&$\times$ & \checkmark \\
    PTS &\checkmark& \checkmark & $\times$ &\checkmark \\
     APPNP-noe&\checkmark&\checkmark&\checkmark & $\times$ \\
    APPNP &\checkmark&\checkmark&\checkmark &\checkmark \\
   
    \hline
    \end{tabular}%
    }
    \label{ab_appnp}
\end{table}

\textbf{Effect of label propagation.} From Table~\ref{ab_appnp_exp}, we can find all the methods with label propagation outperform MLP by a large margin. Specifically, the average accuracy improvement of APPNP (or even if its non-ensemble version APPNP-noe) over MLP on four datasets is 7.40\% (or 5.45\%), which are rather significant.
Moreover, to our surprise, the simple model PTS achieves impressive performance. Although its performance is relatively inferior to APPNP, the values are pretty close. This result validates the major reason of the effectiveness of APPNP is data augmentation with label propagation. 

\begin{table}[t!]
    \centering
    \caption{Performance of APPNP and its variants.}
    \resizebox{.45\textwidth}{!}{%
    \begin{tabular}{ccccc}
    \hline
    Method           & CITESEER     & CORA\_ML     & PUBMED   &  MS\_ACA     \\
    \hline
    MLP&${ 63.98\pm0.44 } $&${68.42\pm0.34}$&$69.47\pm0.47$ & $89.69\pm0.10 $\\
    APPNP-noa-nof&${72.71\pm0.55} $&${78.51\pm0.46}$&$77.18\pm0.53$ & $90.18\pm0.23 $ \\
    PTS & ${ 75.58\pm0.25 } $ & ${  85.02\pm0.24} $&  $  79.67\pm0.28 $ & $92.76\pm0.10$ \\
   APPNP-noe&${ 70.98\pm 0.34 } $&${77.74 \pm 0.27}$&$74.80\pm0.43$ &$89.85 \pm 0.09$  \\
    APPNP &${ 75.48\pm0.29 } $&${85.07\pm0.25}$&$79.61\pm0.33$ &$93.31\pm0.08$  \\
   
    \hline
    \end{tabular}%
    }
    \label{ab_appnp_exp}
\end{table}

\textbf{Effect of weighting.} On the one hand, by comparing PTS with APPNP-noa-nof, we can conclude that weighing the pseudo-labeled data with graph proximity $\bar a_{ji}$ is highly useful. On the other hand, we observe APPNP relatively performs better than PTS as it introduces an additional dynamic weighting $f_{i,h(j)}$, 
which can mitigate adverse effects of structure noise. 

To gain more insight in the effect of dynamic weighting, we conduct an additional experiment to explore the robustness of these methods to the structure noise. Here we first define the \textit{structure noise rates} as the percent of the ``noisy'' edges in the graph, where  the ``noisy'' edges denote the edges whose connected nodes belong to different classes. We randomly transfer some good (noisy) edges into noisy (good) edges to generate simulated graphs with different structure noise rates. The performance of APPNP, PTS, GCN, MLP are presented in Figure~\ref{structure_noise}. For APPNP and PTS, here we choose 2-layer graph propagation for fair comparison with GCN. The black dash line denotes the structure noise rates of the original graph. 

As we can see from Figure~\ref{structure_noise}, when the graph is relative clean (less than 5\%), all the graph-based methods perform pretty well. But it is impractical due to the collected graph are always not so clean. As the structure noise increases, the performance of original GCN drops quickly, which vividly validate the superiority of decoupled GCN over vanilla GCN. Also, we observe the margin achieved by APPNP over PTS become larger with the noise increasing. This experimental result is consistent with our analysis presented in section 3.3 that model-based dynamic weighting indeed improve model's robustness to structure noise. 

More interestingly, when the graph structure noise is large enough (over 50\% in the four data-sets), all GNN-based models (GCN, APPNP, PTS) perform even worse than MLP. 
This phenomenon validates the basic assumption of GCN-based methods is \textit{local homogeneity} \cite{DBLP:conf/icml/YangCS16}, \ie connected nodes in the graph tend to share similar properties and same labels. When the assumption is not hold in the graph, this kind of methods may not be applicable.

\begin{figure}[t!]
    \centering
    \includegraphics[width=0.45\textwidth]{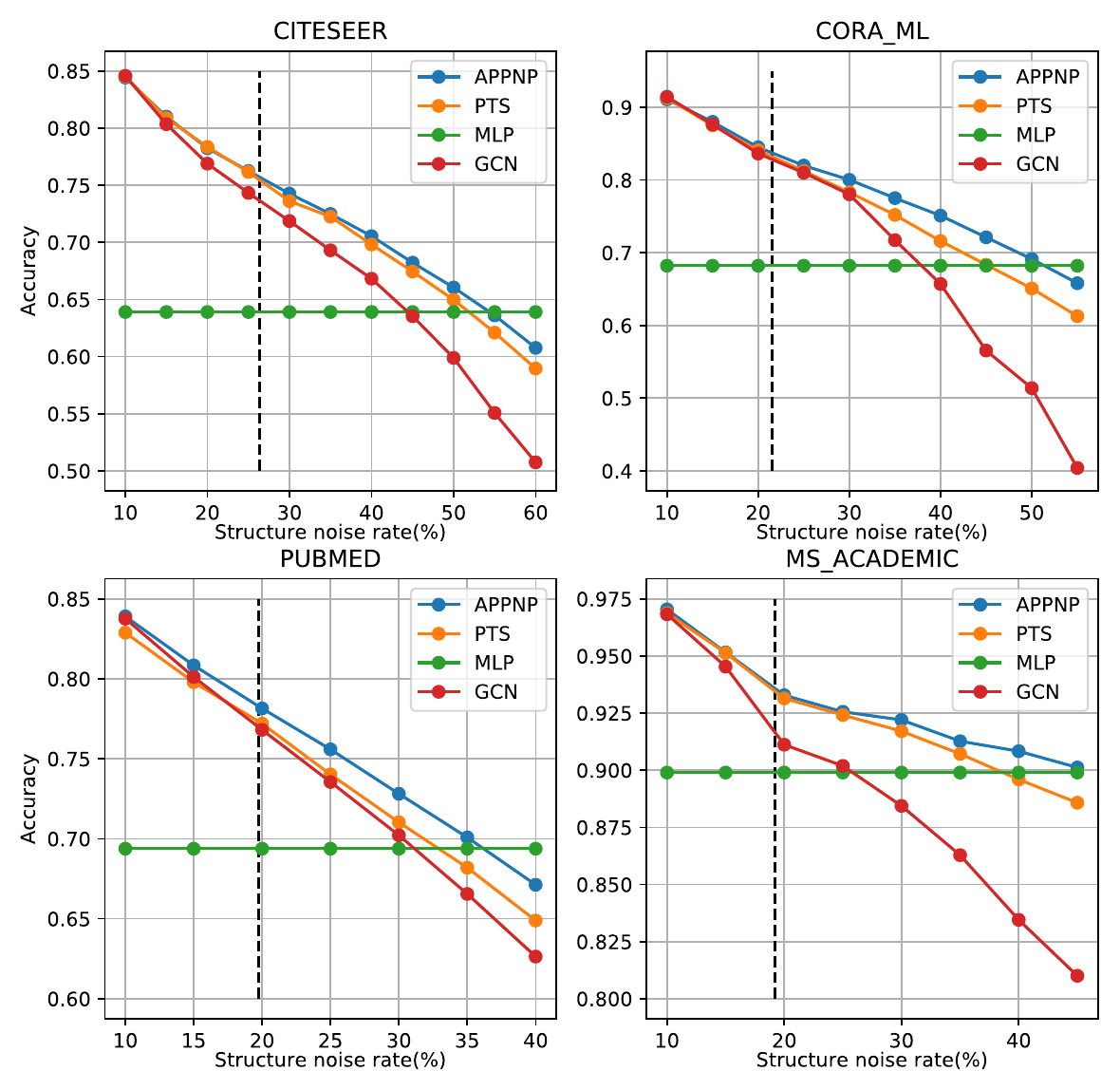}
    \caption{The model robustness to graph structure noise of different models.
    The black dashed line indicates the structure noise rate of the original graph.}
    \label{structure_noise}
\end{figure}

\textbf{Effect of ensemble.} We also observe that APPNP consistently outperforms APPNP-noe. This result validates that ensemble is anther important factor of the effectiveness of decoupled GCN. 

\subsection{Study of Robustness (RQ2)}
In this subsection, we explore how PTA compares with APPNP in terms of robustness to the initialization and label noise. The robustness to structure noise can refer to figure~\ref{structure_noise_2} in Appendix~\ref{PTA_APPNP_str}. 

\textbf{Robustness to initialization.} Figure~\ref{stablity} shows the accuracy distribution of each model is. The smaller boxes and the less outlines suggest the model is more robust to the initialization. From the Figure~\ref{stablity} and Table~\ref{ab_appnp_exp}, We observe that: 1) APPNP, which introduces the dynamic weighting on pseudo-labeled data, is more sensitive to the initialization than PTS. 2) PTA, which adopts the adaptive weighting strategy, is more stable than APPNP and PTD. 3) PTD, where the normalization of weights is removed, is highly unstable. These results validate the necessary and effectiveness of introducing the adaptive weighting strategy. It can reduce the impact of model initialization and foster model's robustness to the noises. 

\begin{figure}[t!]
    \centering
    \includegraphics[width=0.45\textwidth]{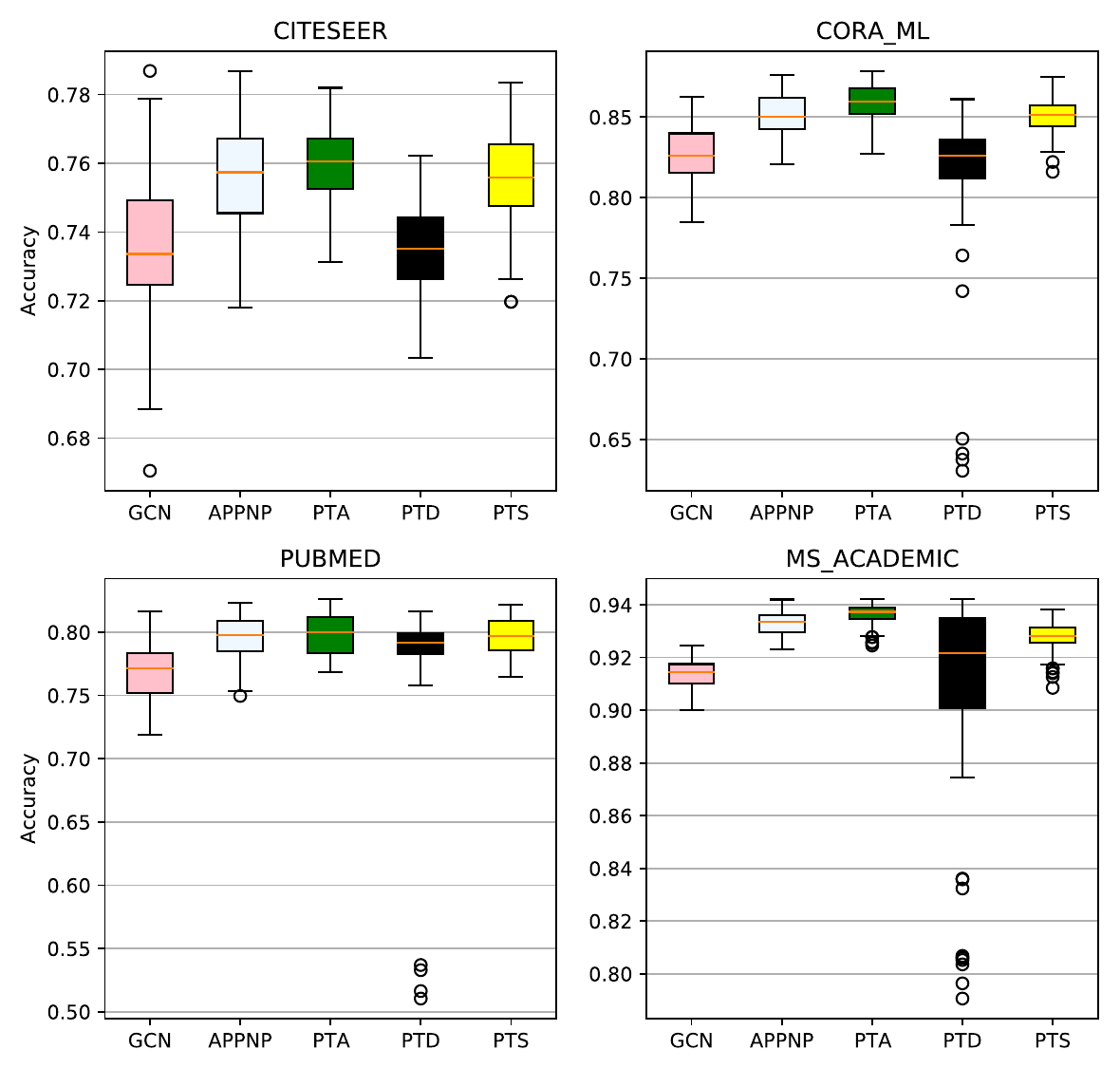}
    \caption{Accuracy distributions of the compared models. }
    \label{stablity}
\end{figure}

\textbf{Robustness to the label noise.} We also conduct an interesting simulated experiments to explore the robustness of the models to the label noise. That is, we randomly transfer a certain number of instances in each class and change their labels to others (``noisy'' labels). We then run different methods on these simulated graph with different ratios of ``noisy'' labels. the result is presented in Figure~\ref{Label_noise}. We observe that: 1) Even there is no label noise, PTA still outperforms others. This result suggests that the equal-contribution assumption is invalid. Different labeled data usually has different quality and importance on learning a neural network predictor. 2) The margins achieved by PTA over APPNP and PTA become larger with the noise increasing. This result is consistent with our analysis presented in section 4 that removing normalization of weights indeed boosts model's robustness to the label noise. 

\begin{figure}[t!]
    \centering
    \includegraphics[width=0.45\textwidth]{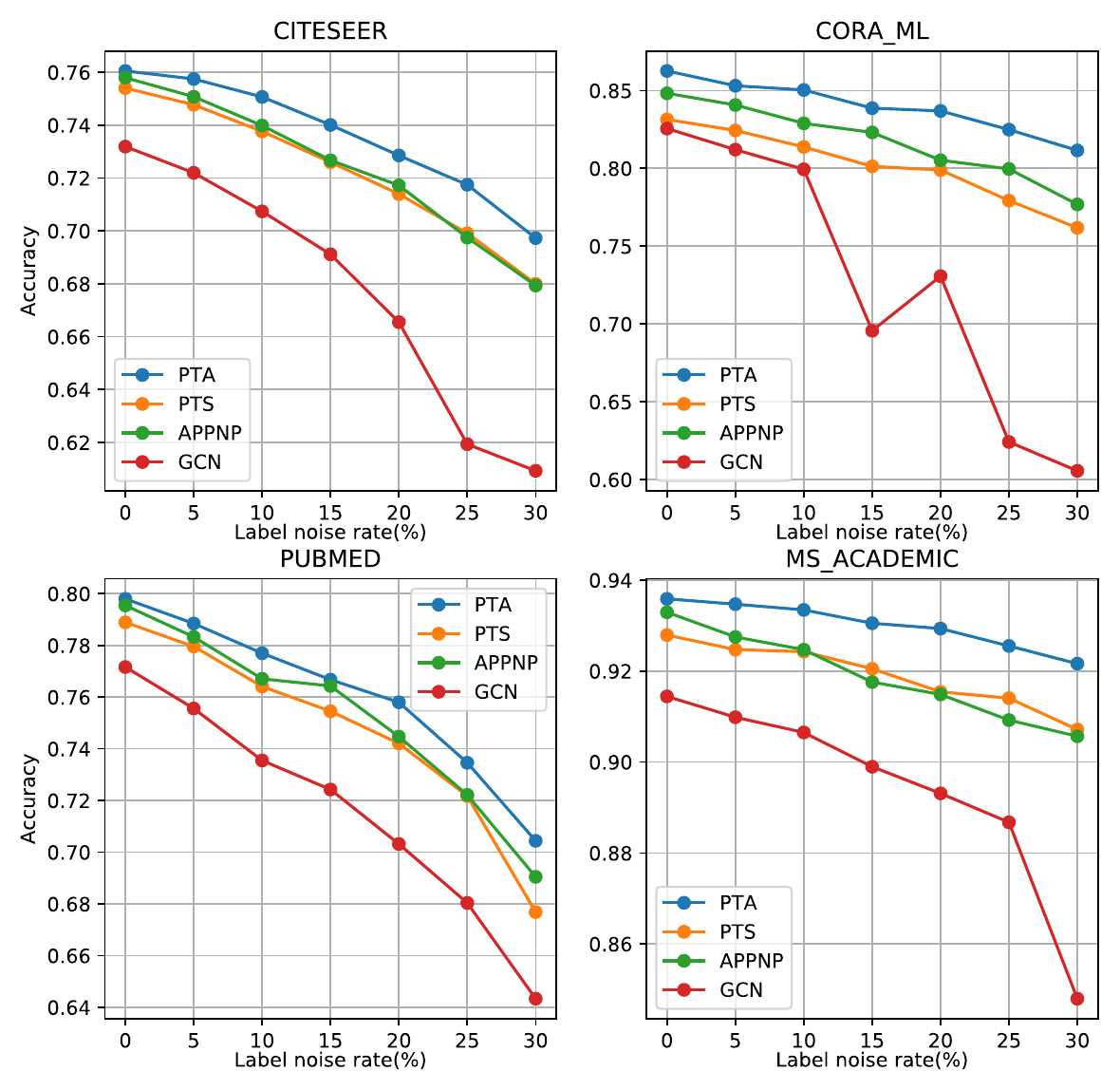}
    \caption{The model robustness to the label noise.}
    \label{Label_noise}
\end{figure}

\subsection{Performance Comparison with State-of-the-Arts (RQ3)}
Table~\ref{accuracy} presents the performance of compared methods in terms of accuracy. The boldface font denotes the winner in that column. For the important baseline APPNP, we present both the results that are reported in the original paper with mark '\#', and our reproduced results. They may have slight difference due to random initialization and random data splits. To ensure the statistical robustness of our experimental setup, we calculate confidence intervals via bootstrapping and report the p-values of a paired t-test between PTA and APPNP. From the table, we have the following observations: 

Overall, PTA outperforms all compared methods on all datasets. This result validates our proposed PTA benefits to train a better neural network predictor. Specifically, comparing with APPNP, the best baseline in general, the improvements of PTA over APPNP is statistically significant with paired t-test at $p < 0.05$ on all datasets.

\begin{table}[t!]
    \centering
    \caption{Accuracy of our PTA comparing with state-of-the-art methods.
    The $p$-value in the last row is obtained via a paired $t$-test between PTA and APPNP.}
    \resizebox{.45\textwidth}{!}{%
    \begin{tabular}{ccccc}
    \hline
    Method            & CITESEER     & CORA\_ML     & PUBMED   &  MS\_ACA    \\
    \hline
    MLP&${ 63.98\pm0.44 } $&${68.42\pm0.34}$&$69.47\pm0.47$ & $89.69\pm0.10 $ \\
    GCN&${ 73.62\pm0.39 } $&${82.70\pm0.39}$&$76.84\pm0.44$ & $91.39\pm0.10 $ \\
    SGCN&${ 75.57\pm0.28 } $&${75.97\pm0.72}$&$71.24\pm0.86$ &$ 91.03\pm0.16 $ \\
    APPNP \#& $75.73\pm0.30$ & $85.09\pm0.25$ & $79.73\pm0.31$ & $93.27\pm0.08$ \\
    DAGNN &${ 74.53\pm0.38 } $&${85.75\pm0.23}$&$79.59\pm0.37$ &$92.29\pm0.07 $ \\
    APPNP&${ 75.48\pm0.29 } $&${85.07\pm0.25}$&$79.61\pm0.33$ &$93.31\pm0.08$  \\
    PTA&$\bm{75.98\pm0.24}$&$\bm{85.90\pm0.21}$&$\bm{79.89\pm0.31}$&$\bm{93.64\pm 0.08}$\\
    \hline
    p-value&$5.56\times 10^{-4}$&$1.81\times10^{-9}$&$1.09\times10^{-2}$&$1.57\times10^{-8}$\\
    \hline
    \end{tabular}%
    }
    \label{accuracy}
\end{table}


\subsection{Efficiency Comparison (RQ4)}
In this section, we empirically compare the efficiency of PTA and APPNP.  Table~\ref{p_time} shows the running time of PTA and APPNP in each epoch, while Table~\ref{t_time} shows the total time cost on training the two models. Note that estimating performance of PTA on the early-stopping set is time-consuming, we further design a fast mode of PTA (PTA(F)), which directly use $f_{\theta}(x)$ instead of ensemble results for early-stopping estimation. The performance of PTA(F) comparing with PTA and APPNP is presented in Table~\ref{accuracy_fast}. From the tables, We can conclude that PTA is much faster than APPNP: on average, about 9.7 times acceleration per epoch and 5.7 times acceleration for totally running time. When we use fast mode (PTA(F)), although its performance would decay slightly (-0.28\% on average), it still outperforms APPNP and achieves impressive speed-up (about 43 times over APPNP and 7.5 times over PTA).

\begin{table}[t]
    \centering
    \caption{The training time per epoch of PTA and APPNP. } 
    \begin{tabular}{ccccc}
    \hline
    Method            & CITESEER     & CORA\_ML     & PUBMED   &  MS\_ACA    \\
    \hline
    APPNP        & 34.73ms & 28.60ms &34.98ms & 30.51ms \\
   PTA         &  $\bm{3.33ms } $   & $\bm{  3.35ms } $&  $ \bm{3.27ms} $ &$\bm{3.33ms} $ \\
    \hline
    \end{tabular}
    \label{p_time}
\end{table}

\begin{table}[t]
    \centering
    \caption{The total training time of PTA and APPNP. 
    }
    \begin{tabular}{ccccc}
    \hline
    Method            & CITESEER     & CORA\_ML     & PUBMED   &  MS\_ACA    \\
    \hline
    APPNP        & 52.75s & 75.30s & 49.39s & 134.23s \\
   PTA         &  $10.14s $   & $11.95s$&  $  10.59s $ & 17.12s  \\
PTA(F)    & $\bm{1.19s} $ & $\bm{ 1.25s }$ & $ \bm{1.40s}$ &$ \bm{3.92s}$ \\
    \hline
    \end{tabular}
    \label{t_time}
\end{table}

\begin{table}[t!]
    \centering
    \caption{The accuracy of PTA(F).}
    \resizebox{.45\textwidth}{!}{%
    \begin{tabular}{ccccc}
    \hline
    Method            & CITESEER     & CORA\_ML     & PUBMED   &  MS\_ACA    \\
    \hline
    APPNP&${ 75.48\pm0.29 } $&${85.07\pm0.25}$&$79.61\pm0.33$ &$93.31\pm0.08$  \\
    PTA(F) & $75.51\pm0.24$ & $85.73\pm0.22 $& $79.45\pm0.40 $&$93.62\pm0.08$ \\
    PTA&$\bm{75.98\pm0.24}$&$\bm{85.90\pm0.21}$&$\bm{79.89\pm0.31}$&$\bm{93.64\pm 0.08}$\\
    \hline
    \end{tabular}%
    }
    \label{accuracy_fast}
\end{table}

%% file: 6_related_work.tex
\section{RELATED WORK}
Inspired by the success of the \textit{convolutional neural networks} (CNN) in computer vision ~\cite{DBLP:conf/cvpr/HeZRS16}, CNN has been generalized to graph-structured data with a so-called \textit{graph convolutional neural network} (GCN). There are two lines to understand GCN: From a spectral perspective, the convolutions on the graph can be understood as a filter to remove the noise from graph signals. The first work on this line was presented by Bruna \etal~\cite{DBLP:journals/corr/BrunaZSL13}.  Defferrard \etal~\cite{DBLP:conf/nips/DefferrardBV16} and Kipf \etal~\cite{kipf2017gcn} further propose to simplify graph convolutions with Chebyshev polynomial to avoid expensive computation of graph Laplacian eigenvectors. 
Afterwards, some researchers put forward their understanding of GCN in spatial domain~\cite{DBLP:conf/iclr/XuHLJ19}. 
From the spatial perspective, the convolution in GCN can be analogized with ``patch operator'' which refines the representation of each node by combining the information from its neighbors. This simple and intuitive thinking inspires much work on spatial GCN~\cite{DBLP:conf/iclr/VelickovicCCRLB18, DBLP:conf/iclr/XuHLJ19, NIPS2017_6703, chen2019measuring}, which exhibits attractive efficiency and flexibility, and gradually becomes the mainstream. For example, Velickovic \etal \cite{DBLP:conf/iclr/VelickovicCCRLB18} leveraged attention strategy in GCN, which specifies different weights to different neighbors; Hamilton \cite{NIPS2017_6703} \etal \cite{DBLP:conf/iclr/XuHLJ19} proposed to sample a part of neighbors for model training to make the GCN scale up to large-scale graph; Xu \etal analyzed the expressive power of GCN to capture different graph structures, and further proposed \textit{Graph Isomorphism Network}; Li \etal \cite{DBLP:conf/aaai/LiHW18} validate the GCN is a special form of Laplacian smoothing and show exsting methods may suffer from over-smoothing issue; Chen \etal \cite{chen2019measuring} further explored over-smoothing issue of GCN and propose two strategies MADGap and AdaEdge to address it. Here we just list most related work. There are many other graph neural models. We refer the readers to excellent surveys and monographs for more details \cite{wu2020comprehensive,cai2018comprehensive}. 

Note that \textit{feature transformation} and \textit{neighborhood aggregation} are two important operations in a spatial GCN model. In the original GCN~\cite{kipf2017gcn} and many follow-up models~\cite{Ma2019, DBLP:conf/sigir/Wang0WFC19, sankar2020beyond}, the two operations are coupled, where each operation is companied with the other in a graph convolution layer. In fact, some recent works find such a coupling design is unnecessary and even troublesome. Klicpera \etal~\cite{klicpera_predict_2019} and Liu \etal~\cite{DBLP:conf/kdd/LiuGJ20dagnn} claim that decoupling the two operations permits more deep propagation without leading to over-smoothing; Wu \etal~\cite{DBLP:conf/icml/WuSZFYW19} and He \etal~\cite{DBLP:conf/sigir/0001DWLZ020} empirically validate the simplified decoupled GCN outperforms the vanilla one in terms of both accuracy and efficiency. Despite decoupled GCN attracts increasing attention and has become the last paradigm of GCN, to our best knowledge, none of work has provided deep analysis of working mechanisms, advantages and limitations about it.

Lastly, both label propagation (LP)~\cite{zhu2002learning} and GCN follow the information propagation scheme, which implies they have the same internal foundation. However, their relations have not been investigated. The most relevant work that jointly considers both methods is~\cite{DBLP:journals/corr/abs-2010-13993}, which trains a base predictor on the labeled data and then corrects it by information propagation on the graph. The angle of this work differs from ours significantly. 
In this work, we prove that decoupled GCN is identical to the \textit{Propagation then Training}, which to our knowledge is the first work that reveals the essential relation between LP and GCN.

%% file: 7_conclusion.tex
\section{Conclusion}

In this work, we conduct thorough theoretical analyses on decoupled GCN and prove its training stage is essentially equivalent to performing \textit{Propagation then Training}. This novel view of label propagation reveals the reasons of the effectiveness of decoupled GCN: 1) data augmentation through label propagation; 2) structure- and model- aware weighting for the pseudo-labeled data; and 3) combining the predictions of neighbors. In addition to the advantages, we also identify two limitations of decoupled GCN --- sensitive to model initialization and to label noise. Based on these insight, we further propose a new method \textit{Propagation then Training adaptive} (PTA), which posters the advantages of decoupled GCN and overcomes its weaknesses by introducing an adaptive weighting strategy. Empirical studies on four node classification datasets validate the superiority of the proposed PTA over decoupled GCNs in all robustness, accuracy, and efficiency.

We believe the insights brought by the label propagation view are inspiration for future research and application of GCN. Here we point out three directions. First, our analyses focus on the semi-supervised node-classification setting. How to extend to other tasks like link prediction and graph classification is interesting and valuable to explore. Second, this work provides a new view of GCN. It will be useful to explore existing methods from this perspective and analyze their pros and cons, which are instructive to develop better models. Third, our proposed PTA uses a relatively simple weighting strategies. More sophisticated weighting strategy can be explored, such as learning from side information, graph global topology, validation data, or employing adversarial learning for better robustness.

%% file: A_APPNP.tex
\section{APPNP to the architecture of decoupled GCN}
The general architecture of \textit{decoupled GCN} can be written:
\begin{equation}
    \begin{aligned}
\hat {\bm Y} = {\mathop{\rm softmax}\nolimits} \left( {\bar {\bm A}{\bm f_\theta }(\bm X)} \right),
 \end{aligned}
    \label{decouple_app}
\end{equation}
and the formulation of APPNP is: 
\begin{equation}
    \begin{aligned}
    &\bm H^{(0)} = \bm f_\theta (\bm X)\\
    &\bm H^{(k)} = (1-\alpha)\hat {\bm A} \bm H^{(k-1)} + \alpha \bm H^{(0)}, \quad k=1,2,...K-1\\
    &\bm {\hat Y} = \operatorname{softmax}\left(\bm H^{(K)}\right).
    \end{aligned}
    \label{APPNP_app}
\end{equation}
Now we prove that APPNP can be subsumed into the architecture of decoupled GCN, with $\bar {\bm A}={(1 - \alpha )^K}{{\hat {\bm A}}^K} + \alpha \sum\nolimits_{k = 0}^{K - 1} {{{(1 - \alpha )}^k}{{\hat {\bm A}}^k}}$. 

\begin{proof}
We prove it by mathematical induction. 

\textit{Base case:} 

When $K=1$, we have $\hat {\bm Y} = \operatorname{softmax}\left([(1-\alpha)\hat {\bm A}\bm  + \alpha \bm I]\bm f_\theta (\bm X)\right) $ and
    ${\bar {\bm A}}^{(1)}= (1-\alpha)\hat {\bm A}  + \alpha \bm I$,  which satisfies $\bar {\bm A}^{(1)}={(1 - \alpha )^K}{{\hat {\bm A}}^K} + \alpha \sum\nolimits_{k = 0}^{K - 1} {{{(1 - \alpha )}^k}{{\hat {\bm A}}^k}}$.
    
\textit{Inductive step:}

Assume the induction hypothesis that for a particular $K\geq 1$ the equations $\hat {\bm Y} = \operatorname{softmax}\left( {\bar {\bm A}}^{(K)}\bm  H^{(0)}\right) $,  $\bar {\bm A}^{(K)}={(1 - \alpha )^K}{{\hat {\bm A}}^K} + \alpha \sum\nolimits_{k = 0}^{K - 1} {{{(1 - \alpha )}^k}{{\hat {\bm A}}^k}}$ hold. Then we have: 
    \begin{equation}
        \begin{split}
            \hat {\bm Y}&=\operatorname{softmax}\left((1-\alpha)\hat {\bm A}  \bm H^{(K)} +  \alpha \bm  H^{(0)}\right) \\
            &= \operatorname{softmax}\left((1-\alpha)\hat {\bm A}  {\bar {\bm A}}^{(K)}\bm  H^{(0)} +  \alpha \bm  H^{(0)}\right) \\
            &=  \operatorname{softmax}\left([(1-\alpha)\hat {\bm A}  {\bar {\bm A}}^{(K)} + \alpha \bm I ]\bm  H^{(0)} \right) \\
            &= \operatorname{softmax}\left([ {(1 - \alpha )^{(K+1)}}{{\hat {\bm A}}^{(K+1)}} + \alpha \sum\nolimits_{k = 0}^{K} {{{(1 - \alpha )}^k}{{\hat {\bm A}}^k}}] \bm f_\theta (\bm X) \right),
        \end{split}
    \end{equation}
which satisfies $\bar {\bm A}^{(K+1)}={(1 - \alpha )^{(K+1)}}{{\hat {\bm A}}^{(K+1)}} + \alpha \sum\nolimits_{k = 0}^{K} {{{(1 - \alpha )}^k}{{\hat {\bm A}}^k}}$. 

Therefore, APPNP can be subsumed into the form of decoupled GCN, with $\bar {\bm A}={(1 - \alpha )^K}{{\hat {\bm A}}^K} + \alpha \sum\nolimits_{k = 0}^{K - 1} {{{(1 - \alpha )}^k}{{\hat {\bm A}}^k}}$. 
\end{proof}

\label{APPNP_decoupled_GCN}

%% file: B_softmax.tex
\section{SOFTMAX function}
\label{softmax}

The original formulation of \textit{decoupled GCN} can be written as:
\begin{equation}
\begin{aligned}
\hat Y = {\mathop{\rm softmax}\nolimits} \left( {\bar A{\bm f_\theta }(X)} \right).
\end{aligned}
\label{decouple_app2}
\end{equation}
The role of the outer ${\mathop{\rm softmax}\nolimits}(\cdot)$ function is to normalize the output into a probability distribution, which is reasonable for optimizing the model with the cross entropy loss.

In our concise formulation of decoupled GCN, the softmax function has been integrated into feature transformation function as follows:
\begin{equation}
\begin{aligned}
{\bm f_\theta }(\bm X)& = \operatorname{softmax}({\widetilde{\bm f}_\theta }(\bm X))\\
\hat {\bm Y} &=  {\bar {\bm A} {\bm f_\theta }(\bm X)}. \\ 
\end{aligned}
\label{decou_app1}
\end{equation}
The major concern is whether its output is reasonable for cross entropy loss. We now show that although the normalization of the output from our concise model may not be hold, it is equivalent to optimize the following objective function, where we give normalized prediction for cross entropy loss.

For arbitrary node $i$ in the training set, the loss function is: 
\begin{equation}
\begin{split}
    l(\hat {\bm y_i}, \bm y_i)&=-\sum\limits_{k\in C} y_{ik} log \left(\hat y_{ik}\right) \\ 
    & = -\sum\limits_{k\in C} y_{ik} log \left( \sum\limits_{j\in V} \bar a_{ij} f_{jk} \right)\\
    & = -\sum\limits_{k\in C} y_{ik} log \left( \frac{ \sum\limits_{j\in V} \bar a_{ij}f_{jk} }{\sum\limits_{q\in V} \bar a_{iq}} \right)  - \sum\limits_{k\in C} y_{ik} log \left(\sum\limits_{j\in V} \bar a_{ij} \right)\\
    & = -\sum\limits_{k\in C} y_{ik} log \left(\sum\limits_{j\in V} g_{ij} f_{jk} \right)  - log \left(\sum\limits_{j\in V} \bar a_{ij} \right),\\
\end{split}
\label{soft_app}
\end{equation}
where $g_{ij} =\frac{\bar a_{ij}}{\sum\limits_{q\in V} \bar a_{iq}}$, satisfying $\sum\limits_{j\in V}g_{ij} = 1$. Thus, we have the normalized prediction $\sum\limits_{k\in C}\sum\limits_{j \in V} p_{ij} f_{jk} = 1$, which meets the constraints. 
We can find the first term of the last line in Equation~(\ref{soft_app}) is a cross entropy loss between the normalized prediction and labels, and the second term is a constant. 



%% file: C_concise_loss.tex
\section{Concise Matrix-Wise Loss Function}
\label{concise_loss}

Overall, PTA optimizes the following objective function:
\begin{equation}
    \begin{aligned}
L_{PTA}(\theta ) = \sum\limits_{i\in V,j \in V_l} w_{ij} \operatorname{CE} \left(\bm f_i, \bm y_j\right) ,\quad \quad {w_{ij}} =\bar a_{ji}{f^\gamma_{i,h(j)}}.
   \end{aligned}
\end{equation}
The equivalent matrix-wise formulation is concise as follows: 
\begin{equation}
    \begin{aligned}
  L_{PTA}(\theta ) & = - \mathop{SUM}\left( \bm Y_{soft} \otimes \bm f(\bm X)^\gamma_{detach} \otimes log\left(\bm f(\bm X)\right)\right).
   \end{aligned}
\end{equation}
where $\otimes$ represents the element-wise product, the subscript ``detach'' denotes no gradient will be backward-propagated along this term, and the $\mathop{SUM}(\cdot)$ represents the sum of all elements on the matrix. 

\begin{proof}
First, let's review the properties of $\bar {\bm A}$ and $\bm y_i$.
\begin{itemize}
    \item [(1)] Note that $\bar {\bm A}$ is calculated from an adjacency matrix $\hat {\bm A}$ with a specific propagation strategy (\ie $\bar{\bm A}=\sum_k \beta_k \hat {\bm A}^k$). In an undirected graph,
     since ${\bm A}$ is a symmetric matrix, we can conclude ${\hat {\bm A}}$ is a symmetric matrix and $\bar {\bm A}$ is symmetric too, \ie $\bar {a}_{ij} = \bar {a}_{ij}$. 
    \item [(2)] $\bm y_j$ is a one-hot vector, where $h(j)$-th element of $\bm y_j$ is 1, \ie $y_{j, h(j)}=1$. 
\end{itemize}
We then have:
\begin{equation}
\begin{split}
L_{PTA}(\theta ) &= \sum\limits_{i\in V,j \in V_l} w_{ij} \operatorname{CE} \left(\bm f_i, \bm y_j\right)\\
&= - \sum\limits_{i\in V,j \in V_l} \bar a_{ji}{f^\gamma_{i,h(j)}} \sum\limits_{k\in C} y_{jk}log \left(f_{ik} \right)\\
&= - \sum\limits_{i\in V,j \in V_l} \bar a_{ji}{f^\gamma_{i,h(j)}}  y_{j,h(j)}log \left(f_{i,h(j)} \right)\\
&= - \sum\limits_{i\in V,j \in V_l} \bar a_{ji} \sum\limits_{k\in C} y_{jk} {f^\gamma_{ik}} log \left(f_{ik} \right) \\
&= - \sum\limits_{i\in V,K \in C}   \left(\sum\limits_{j \in V_l}\bar a_{ji} y_{jk} \right)  {f^\gamma_{ik}} log \left(f_{ik} \right)  \\
&= - \sum\limits_{i\in V,K \in C}   \left(\sum\limits_{j \in V_l}\bar a_{ij} y_{jk} \right)  {f^\gamma_{ik}} log \left(f_{ik} \right)  \\
&=- \mathop{SUM} \left( \bm Y_{soft} \otimes \bm f(\bm X)^\gamma_{detach} \otimes log\left(\bm f(\bm X)\right)\right),\\
\end{split} 
\end{equation} 
where $\bm Y_{soft}$ represents the soft label matrix generated by label propagation. 
\end{proof}

%% file: D_framework.tex
\section{framework of PTA}
\label{framework}
The framework of PTA consists of three parts: data pre-processing, training and inference. 

\paragraph{Data pre-processing} 
We first calculate the soft label matrix $\bm Y_{soft}$ by label propagation. There are various propagation scheme can be chosen. In this paper, we adopt Personalized PageRank for experiments, which can be written as follow: 
\begin{equation}
    \begin{split}
        &\bm Y^{(0)}=\bm Y_{training} \\
        &\bm Y^{(k)}=(1-\alpha) \hat {\bm A} \bm Y^{(k-1)}+ \alpha \bm Y^{(0)}\\
        &\bm y_{i}^{(k)} = \bm y_i, \forall i \in \Set V_l\\
        &\bm Y_{soft} = \bm Y^{(K)}.
    \end{split}
\end{equation}
$\bm Y_{soft}$ will be used in the training stage. 

\paragraph{Training} 
In this step, we train the neural network predictor $\bm f_\theta(\cdot)$ by optimizing the following objective function:
\begin{equation}
\begin{aligned}
L(\theta ) &   = - \mathop{SUM} \left( \bm Y_{soft} \otimes \bm f(\bm X)^\gamma_{detach} \otimes log\left(\bm f(\bm X)\right)\right).\\ 
\end{aligned}
\end{equation}
where $\gamma  = \log (1 + e/\epsilon)$ for PTA,  $\gamma  = 0$ PTS and $\gamma  = 1$ for PTD. Note that we usually use an additional regularization term to avoid over-fitting. 

\paragraph{Inference} 
After training $\bm f_\theta(\cdot)$, following APPNP, ensemble has been adopted for final prediction. 
The formulation is:  
\begin{equation}
    \begin{aligned}
    &\bm H^{(0)} = \bm f_\theta (\bm X)\\
    &\bm H^{(k)} = (1-\alpha)\hat {\bm A} \bm H^{(k-1)} + \alpha \bm H^{(0)}, \quad k=1,2,...K-1\\
    &\bm {\hat Y} = \bm H^{(K)}.
    \end{aligned}
\end{equation}


%% file: E_experimental_details.tex
\section{experimantal details}
\label{detail}

The four datasets used in this paper are downloaded from the official implementation of APPNP~\cite{klicpera_predict_2019} in Github. 
We also follow the data split of APPNP. 
Moreover, we follow the same estimation method for accuracy as APPNP, 
\ie the average performance across 100 different random initializations and uncertainties showing the 95\% confidence level calculated by bootstrapping. All experiments are conducted on a a server with 2 Intel E5-2620 CPUs, 8 2080Ti GPUS and 512G RAM. 

\paragraph{Hyper parameters}  
For fairness, we use the same neural network model scale as the baseline models: two-layer neural network with 64 hidden units. 
We also use the same $K$ and $\alpha$ as APPNP, \ie $K=10, \alpha=0.1$ for three citation graphs, and $K=10, \alpha=0.2$ for co-authorship graph. 
The overall loss function is: $LOSS = \lambda_1 L_1 + \lambda_2 L_2$, where $L_1$ represents loss in Equation~\ref{concise_form}, and $L_2$ represents regularization on the weights of the first neural network layer. 
We fix $\lambda_2 = 0.005$, and find the best $\lambda_1=0.05$. 
We use the Adam optimizer with a learning rate of $lr = 0.1$~\cite{DBLP:journals/corr/KingmaB14}, 
The dropout rate for neural model is $0.0$. 
The addition parameter $\varepsilon$ in Equation(~\ref{pta}) is set as 100 for all datasets. 

%% file: F_struture_noise.tex
\section{Robustness comparison to the Structure Noise}

\label{PTA_APPNP_str}

\begin{figure}[t!]
    \centering
    \includegraphics[width=0.45\textwidth]{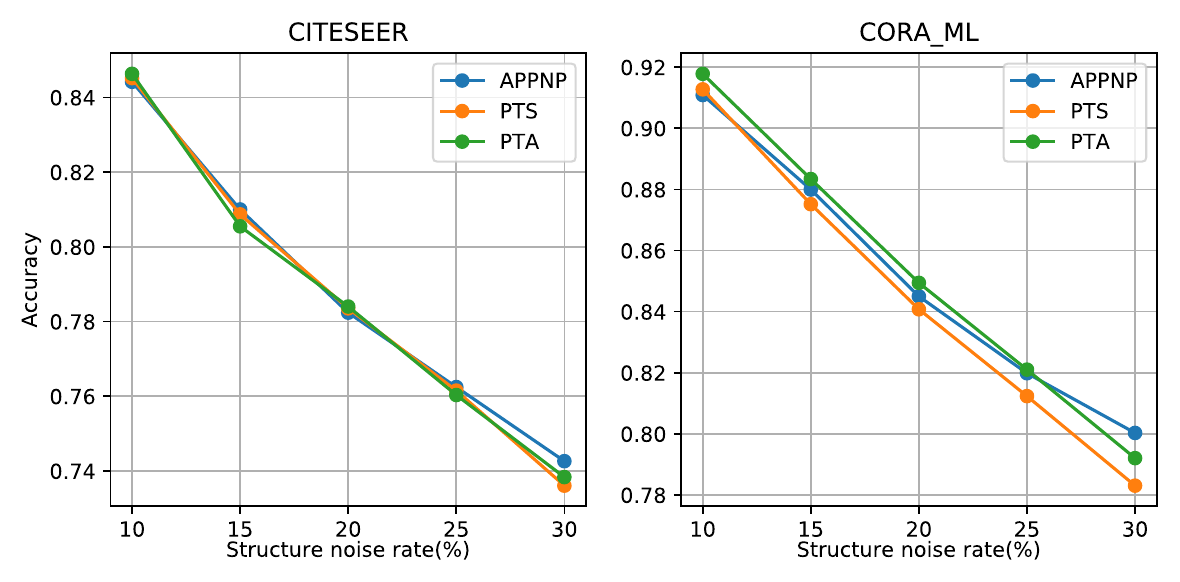}
    \caption{The robustness of different methods to graph structure noise. }
    \label{structure_noise_2}
\end{figure}

Figure \ref{structure_noise_2} shows the structure noise of our PTA comparing with APPNP and PTS. As the performance of these methods are so close, here we just report the noise rate from 10\% to 30\% to amplify the difference. Also, the noise rate of the real-world graph is always on this region. Generally speaking, we can find the performance of PTA and APPNP are better than PTS. Comparing PTA with APPNP, their performance are in the same level. But PTA performs slighter worse than APPNP when the graph has high structure noise rate (\eg rate=30\%). 